\documentclass[letterpaper,12pt]{article}

\usepackage{fullpage}

\usepackage{amssymb}
\usepackage{amsmath}
\usepackage{colortbl}
\definecolor{webgreen}{rgb}{0,0.4,0}
\definecolor{webbrown}{rgb}{0.6,0,0}
\definecolor{purple}{rgb}{0.5,0,0.25}
\definecolor{darkblue}{rgb}{0,0,0.7}
\definecolor{darkred}{rgb}{0.7,0,0}
\usepackage{hyperref}
\hypersetup{colorlinks,citecolor=darkblue,filecolor=black,linkcolor=darkred,urlcolor=webgreen,pdfpagemode=None,
pdfstartview=FitH}

\newcommand{\ignore}[1]{}

\newtheorem{theorem}{{Theorem}}
\newtheorem{definition}{{Definition}}

\newtheorem{claim}{{\sc Claim}}
\newtheorem{example}{{\sc Example}}

\usepackage{cleveref}
\crefname{claim}{claim}{claims}
\crefname{fact}{fact}{facts}
\crefname{algorithm}{algorithm}{algorithms}
\crefname{observation}{observation}{observations}
\crefname{equation}{equation}{equations}
\crefname{assumption}{assumption}{assumptions}

\newenvironment{proof}{\noindent {\em Proof\/}:\enspace}
{\hfill $\blacksquare{}$ \medskip \\}

\DeclareMathOperator*{\argmax}{\arg\!\max}

\newcommand{\val}{\text{\bf val}}

\usepackage{pgf}
\usepackage{verbatim}
\usepackage{enumerate}
\usepackage{amssymb}
\usepackage{mathrsfs}
\usepackage{algorithm}
\usepackage[noend]{algorithmic}
\usepackage{mathtools}

\usepackage{resizegather}
\newif\ifverbose
\verbosefalse %%% Uncomment / Comment this for viewing without / with author comments
\newcommand{\sn}[1]  {\ifverbose {\noindent \textcolor{blue}{{\bf SN: }{\em #1}} } \else \fi }
\newcommand{\is}[1]  {\ifverbose {\noindent \textcolor{red}{{\bf IS: }{\em #1}} } \else \fi }
\usepackage{url}

\usepackage{xspace}

\usepackage{nicefrac}

\newcommand{\mech}{\text{\bf \texttt{SPARCAS}}}
\newcommand{\mechfull}{{\bf \texttt{SP}}ot {\bf \texttt{A}}uction-based {\bf \texttt{R}}obotic {\bf \texttt{C}}ollision {\bf \texttt{A}}voidance {\bf \texttt{S}}cheme}
\usepackage{authblk}
\usepackage[numbers]{natbib}
\usepackage[font=footnotesize]{subfig}
\usepackage{graphicx}
\usepackage{siunitx}
\usepackage{balance}
\usepackage{libertine}

\title{\Large{\bf \mech: A Decentralized, Truthful Multi-Agent Collision-free Path Finding Mechanism}}

\author[1]{Sankar Das}
\author[1]{Swaprava Nath}
\author[1]{Indranil Saha}

\affil[1]{\small Indian Institute of Technology Kanpur, \texttt{\{sdas,swaprava,isaha\}@iitk.ac.in}}

\date{\today}

\begin{document}
\maketitle

\begin{abstract}

\noindent
 We propose a decentralized  collision-avoidance mechanism for a group of {\em independently controlled} robots moving on a shared workspace. Existing algorithms achieve multi-robot collision avoidance either (a)~in a centralized setting, or (b)~in a decentralized setting with {\em collaborative} robots. We focus on the setting with {\em competitive} robots in a {\em decentralized} environment, where robots may strategically reveal their information to get prioritized.
 We propose the mechanism \mech\ in this setting that, using principles of mechanism design, ensures truthful revelation of the robots' private information and provides locally efficient movement of the robots. It is free from collisions and deadlocks, and handles a dynamic arrival of robots. 
 In practice, this mechanism scales well for a large number of robots where the optimal collision-avoiding path-finding algorithm (M*) does not scale. Yet, \mech\ does not compromise the path optimality too much. Our mechanism prioritizes the robots in the order of their `true' higher needs, but for a higher payment.
 It uses {\em monetary transfers} which is small enough compared to the {\em value} received by the robots.
\end{abstract}

\section{Introduction}
\label{sec:intro}

Collision avoidance is a central problem in various multi-agent path planning applications, and the problem has been provided different solutions in different paradigms \cite[e.g.]{snape2010smooth,chen2017decentralized,DesaiSYQS17}. In this paper, we focus on multiple {\em dynamically arriving} and {\em independently controlled} robots to move from their source to their destinations using a track network~\cite{warehouseRoboticsGuizzo,warehouseRobotics,kivaWarehouseRobots}. 
As many robots share the same track network, they are prone to collide with each other.
To avoid that, {\em three} different approaches are employed in the literature.

In the first approach, offline multi-robot planning algorithms are employed to generate the collision-free paths for all the robots statically~\cite[e.g.]{ErdmannRA86,Berg05,JL13,TurpinMK14,WagnerC11,SahaRKPS14, SahaRKPS16}. 
However, this approach has two severe drawbacks. 
First, the {\em computation time} for generating path plans for a large number of robots may be prohibitively large.
Second, they cannot deal with the {\em dynamic arrival} of new robots to the system without recomputing the whole plan. 
%In such a situation, the plan for the existing robots in the system and the newly arrived robots need to be considered
%for replanning, which may cause sever degradation in the performance of the system.

In the second approach, the robots independently generate their trajectories offline without the knowledge about the trajectories of the other robots \cite[e.g.]{Azarm97, Chun99,Jager01,Pallottino04,Olfati-Saber07,Hoffmann08,Purwin08,Velagapudi10,snape2010smooth, Desaraju12, chen2017decentralized}.
Hence, the trajectories of the robots are not collision-free, but are resolved online in a decentralized manner through information exchange among the potentially colliding robots, assuming that the robots will {\em cooperate} with their movements. 
If the simultaneous movements of the robots are not possible,
the robots run a distributed consensus algorithm to find a collision-free plan. 

In the third approach, robots interact in an auction-like setting and bid for their makespan or demand for resource \citep{lagoudakis2005auction,bererton2004auction,nunes2015multi,calliess2011lazy}. Algorithms are designed for certain optimality objectives. However, these approaches do not allow robots to be owned and controlled by independent agents, and therefore do not ensure that these agents {\em bid} their privately observed information (makespan or demand of resource) {\em truthfully} to each other or to the planner. However, the combinatorial auction reduction of the MAPF problem~\cite{amir2015multi} does satisfy truthtelling, but it is centralized and therefore has the same limitations of time complexity and dynamic arrival. The other strand of literature \citep{takei2012time,dhinakaran2017hybrid} on non-cooperative robots consider other robots as dynamic obstacles and solve computationally hard mixed integer programs in a centralized manner.

% Existing algorithms in the literature assume that the robots are {\em collaborative} in nature and they will share their {\em private} information with the other robots truthfully. 
We note that in many commercial settings the robots are controlled by independent operators, e.g., in autonomous vehicle movements, and the robots can potentially {\em manipulate} their bids for a prioritized scheduling. In this paper, we exploit the full strength of mechanism design to ensure truthtelling along with other desirable properties like collision and deadlock avoidance.

\subsection*{Our Approach and Results}
We propose a decentralized mechanism for a group of robots that move on a shared workspace avoiding collision with each other. A collision situation arises when multiple robots try to simultaneously access 
a common location. The competitive robots in our setting engage in a {\em spot auction} and bid the amount they are willing to pay to get access to those locations. 
In our protocol, the robots that `win' the auction get access to move, but for a payment.
The challenge in designing such a protocol is to choose the winners and their payments such that the robots provide their private {\em valuations} for the access locations {\em truthfully}. 

To this end, we consider a quasi-linear payoff model \cite[Chap 10]{SL08} for the robots and propose \mechfull\ (\mech). We show that \mech\ is {\em decentralized, collision-free, deadlock-free}, and {\em robust against entry-exit} (\Cref{obs:general}). The mechanism ensures that every competitive robot reveals its private information truthfully in this spot auction (\Cref{thm:dsic}), does not have a positive surplus of money, and is {\em locally efficient} (\Cref{thm:bb}). 
% {\color{red} SN: Who does not have a positive surplus of money?} 
% We show that \mech\ is also capable of handling dynamic arrival and departure of robots without any significant computational overhead (\Cref{obs:robustness}).

We run extensive experiments in \S\ref{sec:experiment} to evaluate the performance of \mech\ in practice. We show that \mech\ (1)~scales well with large number of robots, (2)~takes much less  planning and execution times compared to that of M*~\cite{WagnerC11} and prioritized planning~\cite{Berg05}, (3)~does differentiated prioritization of different classes of robots, and (4)~handles dynamic arrivals smoothly.

We have simulated our mechanism  for up to 500 robots in Python experimental setup and up to 10 TurtleBots~\cite{turtlebot} using ROS~\cite{ROS}. 
Successful simulation in ROS promises that the proposed mechanism can be implemented in a real multi-robot system.

%The current state-of-the-art methods in collision avoidance uses the information of the sources and destinations of all the robots at once, and a centralized algorithm is used to find the non-colliding path. Our objective is to find a mechanism that is completely decentralized.

%\subsection{Our Approach and Results}

% \subsection{Related Work.}
% 
% \textbf{Contribution.}
% 
% 
% \textbf{Organization.}

\section{Problem Setup}

Define $[n] \coloneqq \{1, \ldots, n\}$. Let $N = [n]$ be the set of robots that are trying to travel from their sources to destinations in a {\em decentralized} manner over a graph $G = (V,E)$, where $V$ is the vertex set and $E$ is the edge set. 
All edges represent full-duplex paths, i.e., robots can travel on both directions on this edge following usual traffic rules. 
Time is discrete and is denoted by the variable $t$. Every edge in the graph is  partitioned into slots where a robot can stand at any given time step. 
Every slot is uniquely numbered between $1, \ldots, S$, where $S$ is the total number of slots. We represent the $k$-th slot by $x_k$, $k \in [S]$.
We call every vertex having a degree of three or more an {\em intersection} point, since robots moving into such vertices cannot guarantee to avoid a collision without coordinating with each other. In practice, an actual intersection will be a roundabout having a fixed number of slots. Every roundabout can hold robots equal to the number of slots at any given time and the movement of all the robots is unidirectional (e.g., counter-clockwise). An illustration of the paths and the roundabout is given in \Cref{fig:roundabout}. A vertex with degree larger than {\em four} can be equivalently represented using a larger roundabout that can accommodate that many number of full-duplex paths.
\begin{figure}[h!]
 \centering
 \includegraphics[width=0.65\linewidth]{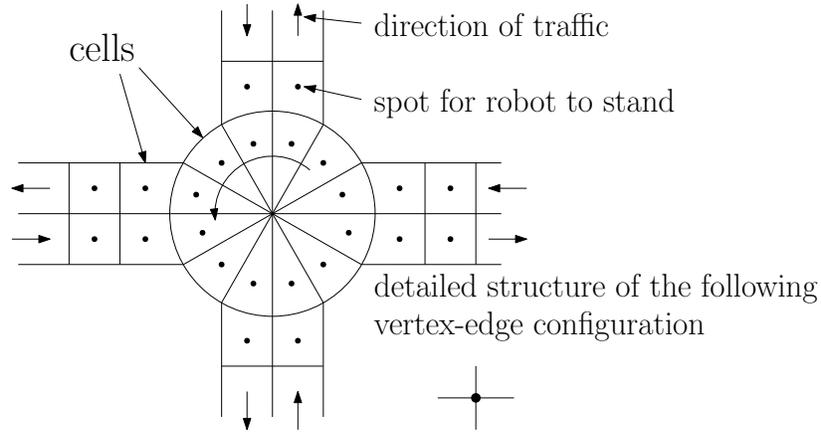}
 \caption{Illustration of a vertex with four incident full-duplex edges.}
 \label{fig:roundabout}
\end{figure}

We assume that these robots are independently controlled (or owned) by agents who have no information about other robots' locations, sources, and destinations. Each robot has a deadline to reach its destination, which contributes to its {\em value} at each time step, e.g., monetary gain from the transfer of a shipment to a customer. A necessary goal of this traversal is to ensure that {\em no collision} occurs between the robots.

In a decentralized traversal plan of the robots, every robot decides its path based on its local information, e.g., those received through its own sensors and shared by other robots in close proximity. Let the path of robot $i \in N$ be denoted by $P_i := (x_{k_1}, x_{k_2}, \ldots, x_{k_{l(i)}})$, where $x_{k_j}$'s are the slots of the graph leading from the source to the destination of $i$. When two different robots reach the same slot at the same time, a collision occurs, which each robot wants to avoid. Collision can occur (a) at a roundabout where two robots intend to move to the same empty cell, or (b) on a cell in the edge where a robot is stationary and another robot moves into the same cell. Let the current location of robot $i$ at $t$ is denoted by $\ell_i(t)$, and the next intended location is $n_i(t)$. 

We assume that a robot's value is calculated internally by the entity that controls it by considering several factors, e.g., its proximity to the deadline, the type of objects it is carrying (high, if it is carrying a priority shipping item), and all such factors are consolidated into a real number $v_i(t)$ for robot $i$ if it can move to its next location $n_i(t)$ at time $t+1$. The value is zero if the robot does not move at $t$.\footnote{We assume this for simplicity of the exposition. However, our algorithm works even when the value is non-zero and has the same time complexity.} 
% Clearly, a robot with a higher value must be prioritized if possible.

In this paper, our primary objective is to find a {\em decentralized mechanism} that (a)~uses the current and intended location of the nearby robots and ensures truthful revelation of their valuations, (b)~avoids collision, deadlocks, (c)~maintains fairness by prioritized scheduling based on higher willingness-to-pay. We assume that the robots are capable of communicating with each other within close proximity and can compute their plans of movement. These objectives are formally defined in \Cref{sec:desiderata}.

\section{Decentralized Mechanisms}
\label{sec:decentralized}

We develop decentralized mechanisms that avoid collision dynamically using the principles of mechanism design theory \cite{borgers2015introduction}. We provide two mechanisms and show how the latter improves the performance even though both of them avoid collision. Every decentralized mechanism discussed in this section is preceded by a computation of the shortest path from the source to destination for every robot.\footnote{There are several polynomial-time algorithms to compute such a path \cite[e.g.]{Hart68}, and we therefore exclude this part from our description of the mechanism. We will, however, include the path computation time in the empirical evaluation to make a fair comparison with other collision-avoidance mechanisms.}

In the current setup, exactly two robots can engage in a collision scenario. A first approach to avoid collision is to develop a protocol where each robot can independently calculate a plan and move according to it. The mechanism emerging from such decentralized planning has to ensure that multiple robots do not appear at the same cell simultaneously. 
% asking the other to stop. The decentralized mechanism in this  can prioritize movement of a robot considering its proximity to the deadline, the type of objects it is carrying (e.g., if it is carrying priority shipping items), and all such factors can be consolidated into a real number $v_i(t)$ for robot $i$ at time $t$. We will call this number the {\em value} of the robot at that instant. Clearly, a robot with a higher value must be prioritized if possible. 

Since, each robot $i \in N$ has a value $v_i(t)$ at every time step $t$, we want the mechanism to yield a decentralized plan for the robots such that it maximizes the sum of their values. Hence, this mechanism is economically {\em efficient}\footnote{In microeconomic theory, the alternative that maximizes the {\em sum} of the values of all the agents is called `efficient', and it is known that it gives rise to a lot of other desirable properties. For a complete discussion, see \cite{borgers2015introduction}.} \cite[Chap 10]{SL08}. In the rest of the paper, we will use the term `efficient' to denote an allocation that maximizes the sum of the valuations of the agents. 
If the robots find out that they are leading to a collision, they share information of their values, e.g., $v_i(t)$ for robot $i$, and an efficient plan of movement is decided to allow the robot with a higher value to move and the other to wait -- breaking ties arbitrarily. We provide a formal definition of `efficiency' in \Cref{sec:desiderata}.

Our first approach is a na\"ive decentralized mechanism, shown in \Cref{algo:naive} at time $t$ for robot $i$ (the mechanism is repeated at every $t$ for every $i$ until each robot reaches its destination).
 \begin{algorithm}[t!]
 \caption{A na\"ive decentralized collision avoidance mechanism for robot $i$}
  \begin{algorithmic}[1]
   \STATE {\bf Input:} cells $\ell_i(t), n_i(t)$, and value $v_i(t)$ of robot $i$
   \STATE {\bf Output:} a decision for robot $i$ to STOP/GO
%    \FOR {every robot $i \in N$}
      \IF {$n_i(t) = \ell_j(t)$, for some $j \neq i$}
       \STATE STOP at $t+1$
      \ELSE 
      \STATE announce the tuple $(\ell_i(t), n_i(t), v_i(t))$ and receive the same from other nearby robots
      \IF {$n_i(t) = n_j(t)$, for some $j \neq i$}
%        \IF {$v_j(t) \geqslant v_i(t)$}
	\IF {\{$v_j(t) > v_i(t)$\} or \{$v_j(t) = v_i(t)$ and $j > i$\}} \label{mech:auction}
	  \STATE STOP at $t+1$
	\ELSE 
	  \STATE GO at $t+1$
	\ENDIF
%       \ELSE
% 	\STATE GO at $t+1$
%       \ENDIF
      \ELSE
	\STATE GO at $t+1$
      \ENDIF
      \ENDIF
%    \ENDFOR
  \end{algorithmic}
 \label{algo:naive}
 \end{algorithm}
 Note that the other robot $j$ that $i$ may collide with, i.e., has $n_i(t) = \ell_j(t)$ or $n_i(t) = n_j(t)$, must be close to robot $i$ and therefore can communicate with $i$ to announce its current and next cell.

 This mechanism is efficient given the value information, i.e., $v_i$'s, of the agents at every $t$ are accurate. Since the robots in our setup are independently controlled, these information are private to the agents. Collision mitigation in such scenarios needs a {\em mechanism} that {\em self-enforces} the agents to reveal it {\em truthfully}. This is important in a setting with competing robots, where a robot can overbid its value to ensure that it is given priority in every collision scenario. The na\"ive mechanism in the current form would fail to ensure efficiency if it uses only the reported values since the strategic robots will overbid to increase their chance of being prioritized.
 
 This is precisely where our approach using the ideas of mechanism design \citep{borgers2015introduction} is useful. 
 In mechanism design theory, it is known that in a private value setup, if the mechanism has no additional tool to penalize overbidding, only degenerate mechanisms, e.g., {\em dictatorship} (where a pre-selected agent's favorite outcome is selected always), are truthful \citep{Gib73,Sat75}. This negative result also holds in the robot collision setting since the Gibbard-Satterthwaite result extends to settings with cardinal values as well. 
%  In a collision scenario at time $t$, robot $i$ gets a value $v_i(t)$ for being prioritized over the other robot(s), and zero if it is asked to wait. 
 A complementary analysis by \citet[Thm 7.2]{Roberts79} shows that a dictatorship result reappears under certain mild conditions in a {\em quasi-linear} setting (which is our current setting and is formally defined later), unless monetary transfers are allowed. Money, in these settings, is used as a means of transferring value from an agent to another, and it is shown to help by unraveling the true values of the agents. In this paper, we, therefore, use monetary transfers among the agents to ensure truthful value revelation at every round. 
 
 \smallskip \noindent
 {\em Robot payoff model}: At every time step $t$, given the current position of the robots, denote the set of feasible next-time step configurations by $A(t+1)$. Hence, for every feasible configuration $a \in A(t+1)$, robot $i$'s valuation is given by $\val_i(a, v_i(t)) = v_i(t)$ if robot $i$ is allowed to move to $n_i(t)$ under $a$, and zero otherwise. The mechanism also recommends robot $i$ to pay $p_i(t)$ amount of money -- which can be negative too, meaning the robot {\em is paid}. The net payoff of robot $i$ is given by the widely used {\em quasi-linear} formula \cite[Chap 10]{SL08}
 \begin{equation}
  \label{eqn:payoff-model}
  \val_i(a, v_i(t)) - p_i(t).
 \end{equation}
 A strategic robot reports its private information $v_i(t)$ to maximize its net payoff. We assume that the robots are {\em myopic}, i.e., consider maximizing their immediate payoff. A mechanism in such a setting is defined as follows.
 \begin{definition}[Mechanism]
  A mechanism $(\mathbf{f} \coloneqq [f_t]_{t = 1,2,\ldots},\mathbf{p} \coloneqq [p_{i,t}]_{i \in N, t = 1,2,\ldots})$ is a tuple which decides the {\em allocation} and {\em payment} for every robot at every time step $t$. Here the allocation function at time $t$ is given by $f_t: \mathbb{R}^n \mapsto A(t+1)$ that decides the next location of every robot $i \in N$, and $p_{i,t} : \mathbb{R}^n \mapsto \mathbb{R}$ denotes the payment made by the robot. Hence, $f_t(v_1(t),\ldots,v_n(t))$ and $p_{i,t}(v_1(t),\ldots,v_n(t))$ denote the allocation and payment of robot $i$ respectively at time step $t$.
 \end{definition}
 
 We assume that the robots agree to a mechanism $(\mathbf{f},\mathbf{p})$ and compute them locally and follow the movement given by $f$ and pay according to $p$ to a central authority. In this paper, we consider allocation and payment functions to be stationary, i.e., independent of $t$.
 
 Under this setup, the na\"ive decentralized mechanism (\Cref{algo:naive}) can be made truthful by adding the following {\em payment rule}. The robot that is prioritized pays the amount equal to the bid of the robot that stops at that time. This is the well-known {\em second price auction} \cite{Vick61}, which is known to be truthful, and we run this at every instant and for every potentially colliding pair of robots.

 We notice that in \Cref{algo:naive}, the robots do not simultaneously move to the same cell and therefore successfully avoid collision in a decentralized manner. However, this mechanism can lead to the following {\em deadlock} scenario. Suppose the intersection is of four full-duplex paths with four cells at the intersection (\Cref{fig:deadlock}). If there  are two robots at the intersection and two more are attempting to enter, and both the entering robots wins the auction step (Step~\ref{mech:auction}) in \Cref{algo:naive}, it will lead to the four cells at the intersection occupied with four robots. If these robots' next locations are the current locations of the robots already in the intersection, none of them can move under this mechanism (see \Cref{fig:deadlock}).
\begin{figure}[h!]
 \centering
 \includegraphics[width=0.50\linewidth]{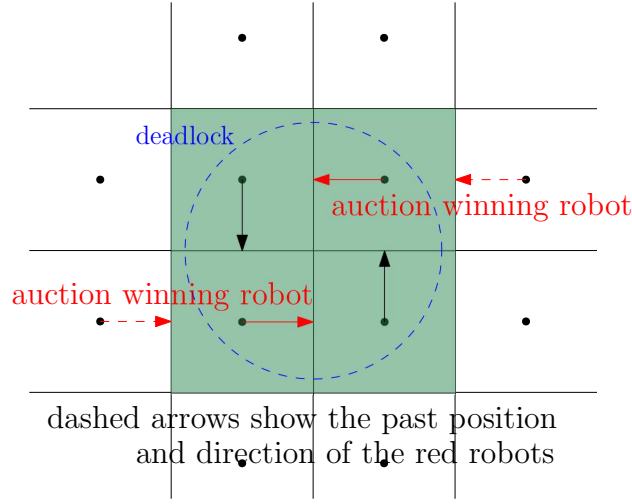}
 \caption{Example of a deadlock scenario. Arrows denote their intended direction of move.}
 \label{fig:deadlock}
\end{figure}

To avoid such a deadlock, a mechanism cannot allow more than $(m-1)$ robots to enter an intersection of capacity $m$. Our proposed decentralized mechanism \mech\ (\mechfull) considers the quasi-linear payoff model, takes account of the known facts of mechanism design, and modifies the na\"ive mechanism to avoid such a deadlock along with other desirable properties. 
For a given graph $G$, let the set of indices of the intersection vertices be represented by $K$. Denote the cells of intersection vertex $k \in K$ by $I_k(V)$. \mech\ modifies how the robots enter an intersection. At every time step $t$, for intersection $k$, it finds the set of robots $N_k(t)$ that are either inside the intersection or are attempting to enter the intersection. From the information shared by all the robots $(\ell_i(t), n_i(t), v_i(t)), i \in N_k(t)$, the mechanism allows every robot $i \in N_k(t)$ to find the feasible next-step configurations $A_k(t+1)$ at that intersection. The feasible configurations ensure that there are no more than $(m-1)$ robots in the intersection of capacity $m$. The set of feasible next-step configurations $A(t+1)$ is an union of the individual feasible configurations $A_k(t+1)$ at each intersection $k \in K$. 
At the non-intersection cells, every robot advances to a cell in its shortest path unless that is already occupied by robot -- we keep these configurations out of the allocations (e.g., $A_k(t)$'s) since we do not need a collective decision there.
The proposed mechanism picks the configuration
\begin{equation}
\label{eqn:efficient}
 a_k \in \argmax_{a \in A_k(t+1)} \sum_{i \in N_k(t)} \val_i(a, v_i(t)).
\end{equation}
% where $\val_i(a, v_i(t)) = v_i(t)$ if robot $i$ is allowed to move to $n_i(t)$, and zero otherwise. 
The configuration maximizes the {\em social welfare} (sum of the values of all the agents) of all the robots that are either inside or are entering the intersection. Denote this configuration by $a_k^*(t+1)$. Similarly, we can define a {\em welfare maximizing} configuration excluding robot $i$ as follows.
\begin{equation}
 \label{eqn:externality}
 a_k^{N_k(t) \setminus \{i\}} \in \argmax_{b \in A_k^{N_k(t) \setminus \{i\}}(t+1)} \sum_{j \in N_k(t) \setminus \{i\}} \val_j(b, v_j(t)).
\end{equation}
This choice maximizes the social welfare of all the robots except $i$. Denote this configuration by $a_k^{*,N_k(t) \setminus \{i\}}(t+1)$. Define the following expression for payment
\begin{align}
 p_i(t) &:= \sum_{j \in N_k(t) \setminus \{i\}} \val_j(a_k^{*,N_k(t) \setminus \{i\}}(t+1), v_j(t)) \nonumber \\ & \hspace{2cm} - \sum_{j \in N_k(t) \setminus \{i\}}  \val_j(a_k^*(t+1), v_j(t)) \label{eqn:payment}
\end{align}
We are now ready to present our proposed mechanism \mech.
For agent $i$ at time $t$, it is described in \Cref{algo:collision}. As before, the mechanism is repeated at every $t$ for every $i$ until each robot reaches its destination.
 \begin{algorithm}[t!]
 \caption{\mech\ for robot $i$ at time $t$}
 \begin{algorithmic}[1]
   \STATE {\bf Input:} cells $\ell_i(t), n_i(t)$, and value $v_i(t)$ of robot $i$
   \STATE {\bf Output:} a decision for robot $i$ to STOP/GO
%    \FOR {every robot $i \in N$}
    \IF {$\ell_i(t), n_i(t) \notin I_l(V), \ \forall l \in K$} %\comment{neither the current nor the next cell is in the intersection}
      \IF {$n_i(t) = \ell_j(t)$, for some $j \neq i$}
       \STATE STOP at $t+1$
      \ELSE 
       \STATE GO at $t+1$
      \ENDIF
	\ELSE 
	  \STATE announce the tuple $(\ell_i(t), n_i(t), \hat{v}_i(t))$  (reported value $\hat{v}_i(t)$ {\em can} be different from the {\em true} value $v_i(t)$) and receive the same from other nearby robots
	  \STATE consider announced tuples of other agents who are in/heading to the same intersection: $(\ell_j(t), n_j(t), \hat{v}_j(t)), \ j \in N_k(t) \setminus \{i\}$, where $k$ is s.t. $n_i(t) \in I_k(V)$
      \STATE compute $a_k^*(t+1)$ (\Cref{eqn:efficient}) and STOP/GO according to that recommendation in time $t+1$ \label{alg:computation-step}
      \STATE pay $p_i(t)$ (\Cref{eqn:payment}) to a trusted authority (e.g., warehouse manager)
% 	    \IF {\{$\hat{v}_j(t) > \hat{v}_i(t)$\} or \{$\hat{v}_j(t) = \hat{v}_i(t)$ and $j > i$\}}
% 	      \STATE STOP at $t+1$
% 	    \ELSE
% 	      \STATE GO at $t+1$
% 	      \STATE PAY $\hat{v}_j(t)$ to the designer
% 	    \ENDIF
	\ENDIF
%    \ENDFOR
  \end{algorithmic}
 \label{algo:collision}
 \end{algorithm}

\Cref{algo:collision} describes the mechanism from the individual agents' point of view. The consolidated payment collected by the trusted authority at intersection $k$ at $t$ from all the robots at that intersection is distributed equally to the robots that were not part of intersection $k$ at $t$. Therefore, \mech\ does not accumulate money from the agents.

\smallskip \noindent
{\em Locally centralized via intersection manager}: \mech\ does not need any synchronization among the robots except that they all follow a common clock (which is also a standard assumption in robot collision avoidance literature \cite[e.g.]{WagnerC11,WAGNER20151,Berg05}). Yet, in \Cref{sec:theory}, we show that it satisfies many desirable properties. However, there is some computational redundancy in \mech. Step~\ref{alg:computation-step} of \Cref{algo:collision} is computed by every robot involved in the mechanism near an intersection point, and this makes the mechanism completely decentralized. A presence of an intersection manager (a trusted intermediary) at $I_k(V)$ who can collect the reported $(\ell_i(t), n_i(t), \hat{v}_i(t)), \forall i \in N_k(t)$ and compute it once and inform all the robots in $N_k(t)$ could substantially reduce the cost of computation of every robot. Though this will come at a slight compromise in the decentralized feature of \mech, we want to make the reader informed about both the versions of the mechanism.

In the following section, we define certain desirable properties for decentralized robot collision avoidance schemes and show that \mech\ satisfies all of them. 

\section{Design Desiderata}
\label{sec:desiderata}

In a decentralized robot path planning, the agents choose their own route from the source to the destination. The collision avoidance mechanism ensures a protocol that is applied locally at a potential colliding scenario. The property desirable in such a setting is of {\em locally efficient} prioritization.
\begin{definition}[Local Efficiency]
 \label{def:efficiency}
 A robotic collision avoidance mechanism $(\mathbf{f},\mathbf{p})$ is {\em locally efficient} if for every time $t$, every intersection $k \in K$, it chooses an allocation that maximizes the sum value of all the robots. Formally, it picks $\forall t, \forall k \in K$
 \[f(v_i(t), i \in N_k(t)) \in \argmax_{a \in A_k(t+1)} \sum_{i \in N_k(t)} \val_i(a, v_i(t)).\]
\end{definition}

However, in a multi-agent setting, $v_i(t)$'s of the robots are unknown to the mechanism, which can only access the reported values $\hat{v}_i(t)$'s. Therefore, the following property ensures that the robots are incentivized to `truthfully' reveal these information.
\begin{definition}[Dominant Strategy Truthfulness]
 \label{def:non-manipulability}
 A mechanism $(\mathbf{f},\mathbf{p})$ is {\em truthful in dominant strategies} if for every $t$, $v_i(t), \hat{v}_i(t), \hat{v}_{-i}(t)$\footnote{We use the subscript ${-i}$ to denote all the agents except agent $i$, therefore, $v_{-i}:=(v_1, \ldots, v_{i-1},v_{i+1}, \ldots, v_n)$.}, and $i \in N$
  \begin{align*}
   \lefteqn{ \val_i(f(v_i(t), \hat{v}_{-i}(t)), v_i(t)) - p_i(v_i(t), \hat{v}_{-i}(t))} \\
   \quad &\geqslant \val_i(f(\hat{v}_i(t), \hat{v}_{-i}(t)), v_i(t)) - p_i(\hat{v}_i(t), \hat{v}_{-i}(t)).
  \end{align*}
\end{definition}
The inequality above shows that if the true value of robot $i$ is $v_i(t)$, the allocation and payment resulting from reporting it `truthfully' maximizes its payoff irrespective of the reports of the other robots.

Since we consider mechanisms with monetary transfer, an important question is whether it generates a surplus amount of money. The following property ensures that there is neither surplus nor deficit.
\begin{definition}[Budget Balance]
 A mechanism $(\mathbf{f},\mathbf{p})$ is {\em budget balanced} if for every $t$ and $(v_i(t),v_{-i}(t))$,  $\sum_{i \in N} p_i(v_i(t),v_{-i}(t)) = 0$.
\end{definition}

 In the context of decentralized robot path planning, a mechanism that is robust against robot failures is highly desirable. Also, it is desirable if a robot can start its journey when other robots are already in motion, and the mechanism does not need other robots to re-compute their path plan. 
\begin{definition}[Entry-Exit Robustness]
 A mechanism $(\mathbf{f},\mathbf{p})$ is {\em robust against entry or exit} of the robots if the properties satisfied by the other robots' path plans are unaffected by an addition or deletion of a robot under the mechanism.
\end{definition}
 Centralized collision avoidance mechanisms that compute the paths and recommends that to all the robots are not robust against entry or exit. With every addition or deletion of a robot, the plan has to be recomputed. It is to be noted that decentralization alone cannot guarantee entry-exit robustness. Decentralization only implies that the decisions are taken independently by each of the robots. It does not restrict the way in which they interact with each other. Based on the interaction, the plan may not be robust against entry-exit, as the following example shows.
\begin{example}[Decentralized but not Entry-Exit Robust]
 Consider the following decentralized version of the prioritized planning algorithm. Before beginning to move, the robots send their identities on a common channel (assume that there is a wired broadcast channel connecting all the starting positions), but this common channel is not available when they are on the move. The highest priority robot computes its path and broadcasts it (the priority order is fixed beforehand and is a common knowledge of all the robots). After listening to that plan, the second highest priority robot plans its path considering the former as a dynamic obstacle and broadcast, and the process continues for robots of the next priorities. After all robots compute their plans in this {\em decentralized} manner, they leave their parking slots and follow their pre-decided (but decentralized) plan of movement. This scheme is clearly not robust against entry-exit. A newly joined robot does not have the earlier broadcast messages and hence cannot plan its path. Also if that robot has a higher priority than some of the robots that are moving already, then those lesser priority robots cannot change their path plan.
\end{example}

 In the following section, we show that our proposed mechanism satisfies all these properties. 
%  We discuss why certain additional desirable properties cannot be satisfied along with these in our setup. 
 In \Cref{sec:experiment}, we consider a real warehouse setting and exhibit the performance of \mech\ in practice.

\section{Theoretical guarantees}
\label{sec:theory}

% From the construction of \mech, the following observations can be made.
% \begin{observation}
% \label{obs:general}
%  \mech\ is collision-free, decentralized, deadlock-free, and locally efficient.
% \end{observation}
% 
% The mechanism allows every robot to independently make their moves. However, if they are closer to an intersection point, then each robot in that local neighborhood does message passing, compute the allocation and then move according to it. We assume that the robots' clocks are synchronized and by the mechanism each of them computes the same allocation which keeps one block in the intersection empty. Therefore, they avoid collision and deadlocks, and is completely decentralized. The allocation of \mech, given by \ref{eqn:efficient}, maximizes the social welfare at every intersection. Therefore, it is locally efficient.

The property of truthfulness is important in the multi-agent setting since it ensures that the allocation decision is taken on the {\em true} values of $v_i(t)$'s and the actual locally efficient allocations were done.
\begin{theorem}
\label{thm:dsic}
 \mech\ is dominant strategy truthful.
\end{theorem}

\begin{proof}
 This proof is a standard exercise in the line of the proof for Vickery-Clarke-Groves (VCG) mechanism \cite{Vick61,Clar71,Grov73}. \mech\ follows the VCG allocation and payment locally at every intersection calculated by the robots independently. Hence, the payoff of robot $i$ at intersection $k$ is given by (for brevity of notation, we hide the time argument in every function and write $a_k^*$ as $a_k^*(v_i, v_{-i})$)
\begin{gather}
 \begin{align*}
   \lefteqn{ \val_i(a_k^*(v_i, \hat{v}_{-i}), v_i) - p_i(v_i, \hat{v}_{-i})} \\
   &= \sum_{j \in N_k}  \val_j(a_k^*(v_i, \hat{v}_{-i}), \hat{v}_j) - \sum_{j \in N_k \setminus \{i\}} \val_j(a_k^{*,N_k \setminus \{i\}}, \hat{v}_j) \\
   &\geqslant \sum_{j \in N_k}  \val_j(a_k^*(\hat{v}_i, \hat{v}_{-i}), \hat{v}_j) - \sum_{j \in N_k \setminus \{i\}} \val_j(a_k^{*,N_k \setminus \{i\}}, \hat{v}_j) \\
   &= \val_i(a_k^*(\hat{v}_i, \hat{v}_{-i}), v_i) - p_i(\hat{v}_i, \hat{v}_{-i}).
  \end{align*}
\end{gather}
The first equality is obtained by writing and reorganizing the expression for $p_i$. The inequality holds since by definition $\sum_{j \in N_k}  \val_j(a_k^*(v_i, \hat{v}_{-i}), \hat{v}_j) \geqslant \sum_{j \in N_k}  \val_j(a_k, \hat{v}_j)$ for every $a_k$; in particular, we chose $a_k^*(\hat{v}_i, \hat{v}_{-i})$. The last equality is obtained by reorganizing the expressions again.
\end{proof}

\mech\ redistributes the generated money from a particular intersection $k$ to the robots who are not part of it at that time step (see the paragraph following \Cref{algo:collision}). Clearly, this does not affect the truthfulness properties. The generated surplus before redistributing can be shown to be always non-negative.\footnote{The intuition for this claim is that absence of a robot reduces congestion, which leads to the other robots being allowed to move. This increases the sum of the values of the other robots. Therefore, we skip a formal proof of this fact.\label{foot:payment-positive}} The allocation of \mech, given by \Cref{eqn:efficient}, maximizes the social welfare at every intersection. Therefore, it is locally efficient. Hence we get the following theorem.
\begin{theorem}
\label{thm:bb}
 \mech\ is budget balanced and locally efficient.
\end{theorem}

\mech\ satisfies certain properties by construction. The following claim summarizes them and we explain it below.
\begin{claim}
\label{obs:general}
 \mech\ is decentralized, collision-free, deadlock-free, and robust against entry or exit.
\end{claim}
\is{Change the order of decentralized and collision-free?} \sn{done}

In \mech, each robot near an intersection point does message passing, compute the allocation and then move synchronously according to it. Each of the robots at an intersection computes the same allocation which keeps one block in the intersection empty (see the definition of the allocation set $A_k$). Thus the robots avoid collision and deadlock, and the mechanism is completely decentralized. 

Note that \mech\ depends only on the values reported by the robots that are already in an intersection or are about to enter it. 
A newly entered robot can at most take part in such an interaction, but cannot change the way other robots in other intersections interact with each other. Hence, the properties that those robots were satisfying before the entry of this robot continue to be satisfied. Hence we get the claim.

In the following section, we investigate the performance of \mech\ in real-world scenarios.

% \sn{\mech\ is collision-free, truthful, deadlock-free, decentralized, budget-balanced, entry-exit robust, locally efficient. May not be individually rational -- needs to check.}

% \begin{itemize}
%  \item Fault tolerance: can adjust to robot failure dynamically.
%  \item Efficient transmission of objects: maximizes the social welfare among all robots.
%  \item Manipulation proof: no robot owner can gain by misreporting her value for the object at any point.
%  \item No money burning: mechanism redistributes the generated money -- no deficit or surplus.
% \end{itemize}

\section{Experiments}
\label{sec:experiment}

While \mech\ satisfies several desirable properties of a collision avoidance mechanism, its scalability, time complexity to find a collision-avoiding path, and differentiated treatment with different classes of robots are not theoretically captured in the previous sections. This is why an experimental study is called for. Note that the mechanisms that fall in the third set of approaches in \S\ref{sec:intro} either only use the bidding part of the auctions and not the payment which is essential to guarantee truthfulness in an auction~\citep{lagoudakis2005auction,bererton2004auction,nunes2015multi,calliess2011lazy} or is a reduced combinatorial auction of the centralized algorithm~\cite{amir2015multi}. Either way they are incomparable with \mech, which is truthful and decentralized. We compare \mech with two widely used protocols for multi-agent path finding -- (a) M*~\cite{WagnerC11}: optimal, but has a significant time complexity, and (b) prioritized planning~\cite{Berg05}: suboptimal, but has low time complexity.

To evaluate \mech\ experimentally, we use a 2-D rectangular workspace representing a road network.
An example of such a workspace of size $16\times 16$ is shown in \Cref{fig:fig-workspace}. 
A robot picks and delivers an object from and to a cell in the service area. 
Each robot follows the traffic rules in the road network, and moves along with the directions of the arrows from the source to the destination. 
It can change its direction only within an intersection cell, but its options are restricted by the available traffic directions in that particular intersection cell. 
\begin{figure}[!ht]
\begin{minipage}{0.49\linewidth}
 \centering
 \includegraphics[width=0.99\linewidth]{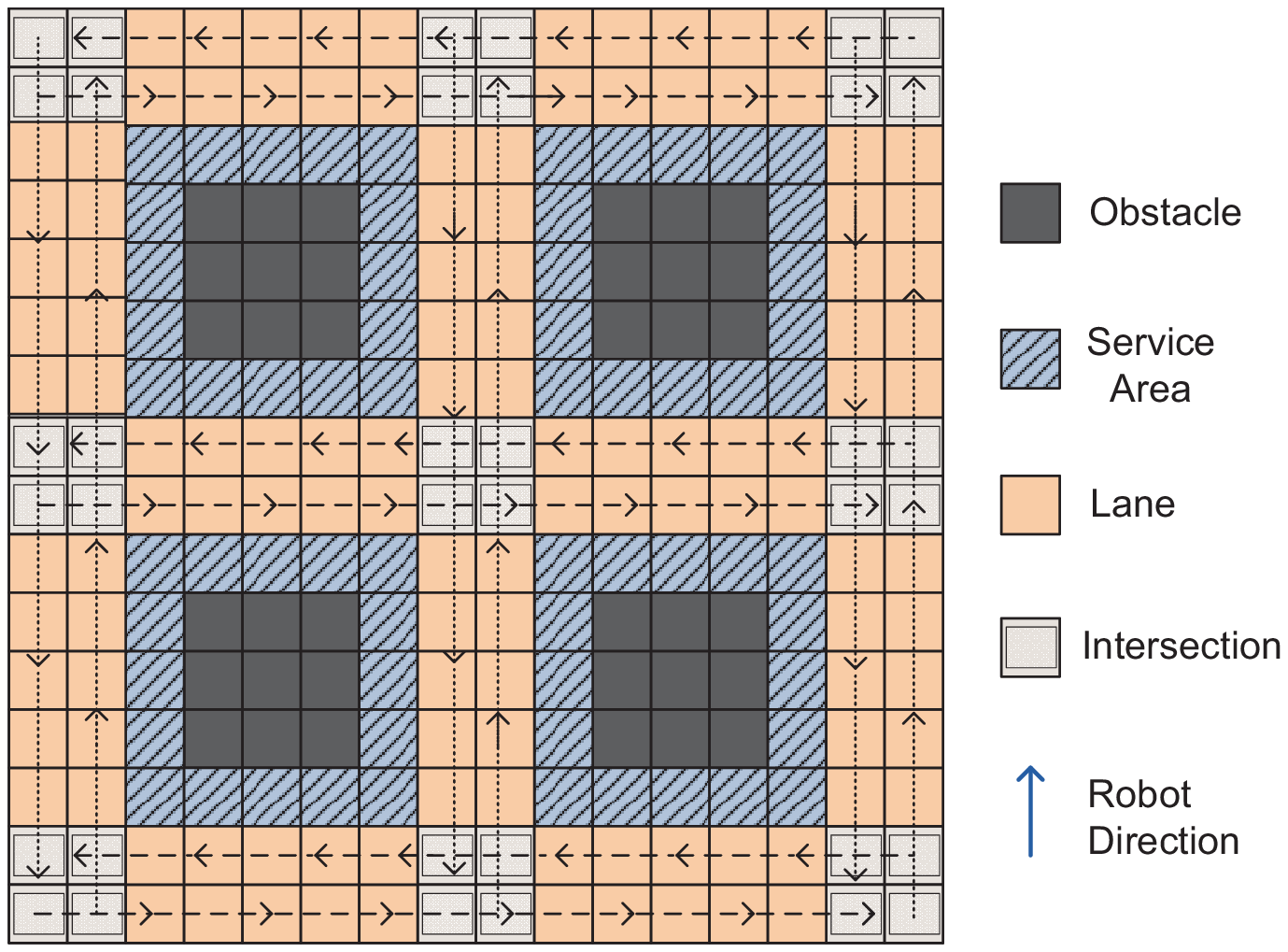}
 \caption{Illustration of a $16\times16$ workspace.}
 \label{fig:fig-workspace}
\end{minipage}
\hspace{0mm}
\begin{minipage}{0.49\linewidth}
\centering
 \includegraphics[width=0.99\linewidth]{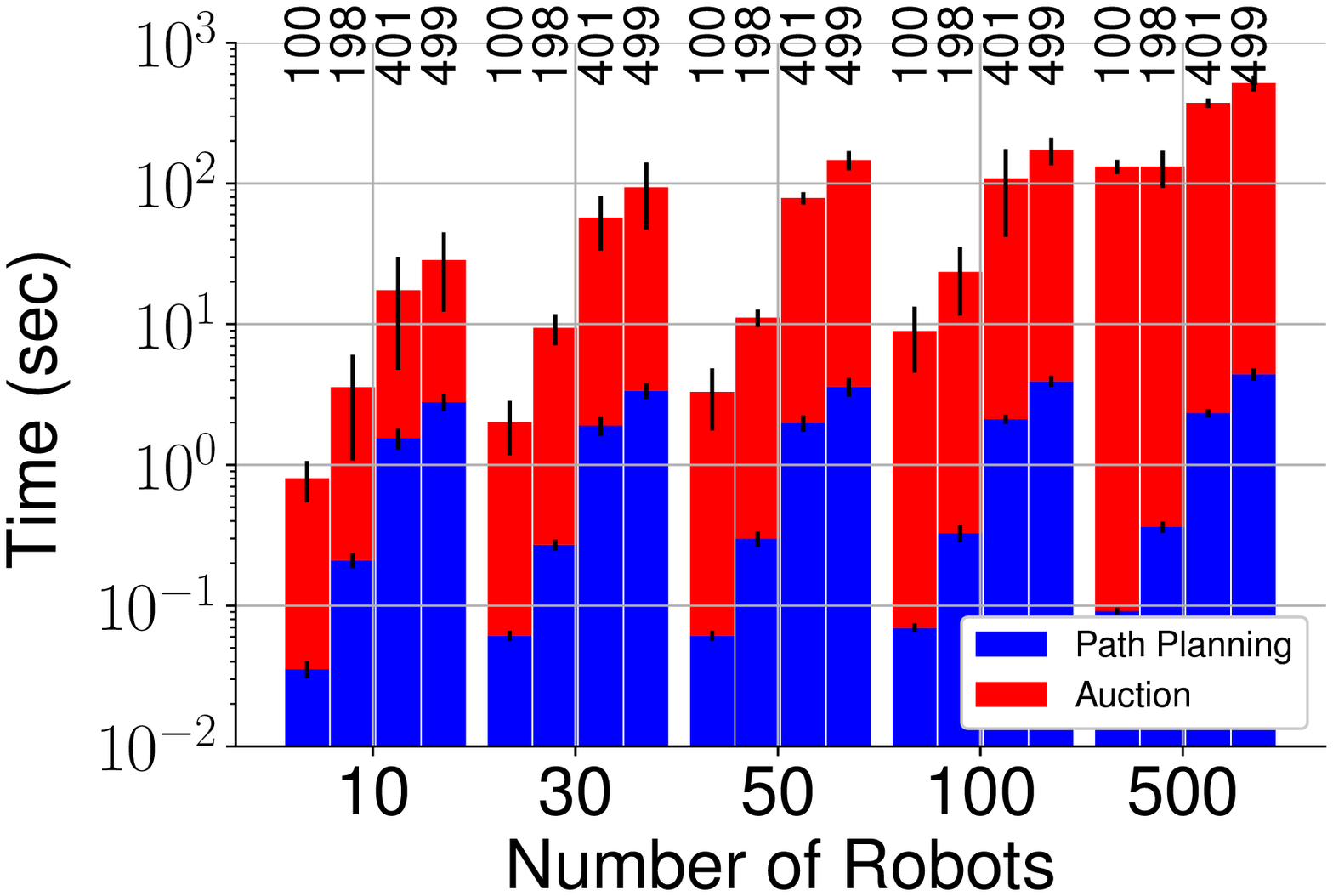}
 \caption{Path planning and auction times of \mech. The numbers on the bars show the workspace width for each number of robots.}
 \label{fig:scalability}
\end{minipage}
\end{figure}

% \begin{figure}[h!]
%  \centering
%  \includegraphics[width=0.65\linewidth]{images/example_workspace_updated_16.eps}
%  \caption{Illustration of a $16\times16$ workspace.}
%  \label{fig:fig-workspace}
% \end{figure}

% 
% {\color{red} TODO: add the size} is shown in . 
% {\color{red} add a description of the workspace.}
% 
% {\color{red} mention the machine configuration where the experiments were executed. 
% mention how many times the  same experiment was carried out.}

We have implemented \mech\ in \textsf{Python}.\footnote{All codes are available at \href{https://bit.ly/2lx0fHk}{\bf https://bit.ly/2lx0fHk}}
The simulations have been performed in a 64-bit Ubuntu 14.04 LTS machine with Intel(R) Core TM i7-4770 CPU @3.40 GHz $\times$ 20 processors and 128 GB RAM. Each run for a specific number of robots and the workspace size is performed 20 times to calculate the average of the result. The source and goal locations of the robots are selected independently and uniformly at random from the cells of the service area.

% The values of the robots are chosen as follows. 
A robot is generated uniformly at random from one of the three classes: {\em economy, regular, premium} having weights 
%$0.05, 0.1, 0.2$
$0.02$, $0.065$, and $0.2$
 respectively.\footnote{Weights are increasing with a rough multiplying factor of $3.25$.} At every time instant of the experiments, the {\em valuation} of a robot entering an intersection is assumed to be $(t_\text{wait} + 1) \times w$, where $t_\text{wait}$ is the wait time of the robot till that instant and $w$ is the weight as described above.

% \setcounter{paragraph}{1}
% \paragraph{\arabic{section}.\arabic{paragraph} Scalability}
\subsection{Scalability} \label{sec:scalability}
We evaluated the scalability of \mech\ on workspaces of four different sizes: 
$100\times100$, $198\times198$, $401\times401$  and $499\times499$,\footnote{These numbers ensure a regular pattern of the workspace of \Cref{fig:fig-workspace}.} and for different number of robots between 
$10$ and $500$. \Cref{fig:scalability} shows how the computation time of \mech\ varies with the number of robots
and the size of the workspace.
The computation time of \mech\ has two components: (a) time for an offline path computation and (b) the time required to run the spot auctions (given by \Cref{algo:collision}) online. 
In our experiments, we assumed that the robots compute their paths to reach the destination using A* algorithm~\cite{Hart68}. As this computation can take place in parallel, we count the {\em maximum} of the plan computation times for all the robots as the offline computation time (part (a)).
We assume that each plan execution slot is preceded by a \emph{mini-slot} when the auction can take place at any intersection. We take the product of the duration of the mini-slot dedicated for auction and the maximum number of steps to finish the plan execution for any robot (the makespan) to find the total time spent in spot auction (part (b)).
The duration for the mini-slot has been decided based on extensive simulation with different reasonably-sized workspaces and robot populations.
% in the ROS environment.
% {\color{red} TODO: Mention how the duration of the mini-slot has been decided.} \sn{shall we also say the relative smallness of the mini-slot to the execution slot? e.g., The maximum length of a mini-slot is about 70 milliseconds, while a robot (some specification of the robots) movement time is around 6 seconds.}

% \begin{figure}[!ht]
% \centering
%  \includegraphics[width=0.79\linewidth]{images/scalability_log_2.eps}
%  \caption{Path planning and auction times of \mech. The numbers on the bars show the workspace width for each number of robots.}
%  \label{fig:scalability}
% \end{figure}

% \stepcounter{paragraph}
% \paragraph{\arabic{section}.\arabic{paragraph} Comparison with Static multi-robot planning algorithms}
\subsection{Comparison with static multi-robot planning algorithms} \label{sec:comparison}
We compare the performance of \mech\ with that of M*~\cite{WagnerC11} and prioritized planning~\cite{Berg05}, two state-of-the-art multi-robot path planning algorithms. 
% Given the source and destination locations of a number of robots in a workspace, M* find collision-free paths for all the robots simultaneously. %efficiently. 
The original version of M* ensures the optimality of the generated multi-robot plan in terms of the total cost (sum of the time for all the robots to reach destination). 
However, a variant of M* called \emph{inflated} M* can produce a sub-optimal plan faster than the original M*.
We compare \mech\ with inflated M* too. 
The prioritized planning algorithm also does not provide any optimality guarantee, but can generate collision-free paths much more time-efficiently than both versions of M*.
%We have set a timeout of $1200\si{s}$. If the computation does not finish before the timeout, we consider the timeout duration as the computation time of \textcolor{red}{M*, inflated M*} algorithms (which is a lower bound).

We compare \mech, M*, inflated M*, and prioritized planning for a $100\times100$ workspace and different number of robots upto 500 robots.
We have set a timeout of $1200\si{s}$. If the computation does not finish before the timeout, we consider the timeout duration as the computation time of the  algorithms (which is a lower bound).
The experimental results are shown in \Cref{fig:comparison-with-m*-sync}. 
The lines connect through the average values of the planning times, while the actual planning times for $20$ runs for every number of robots are scattered according to their values in the figure. The results show that \mech \ outperforms all the other three algorithms in terms of computation time.  M* and inflated M* algorithms do not scale beyond 75 robots for this timeout.

\begin{figure}[!ht]
\begin{minipage}{0.49\linewidth}
 \centering
 \includegraphics[width=0.99\linewidth]{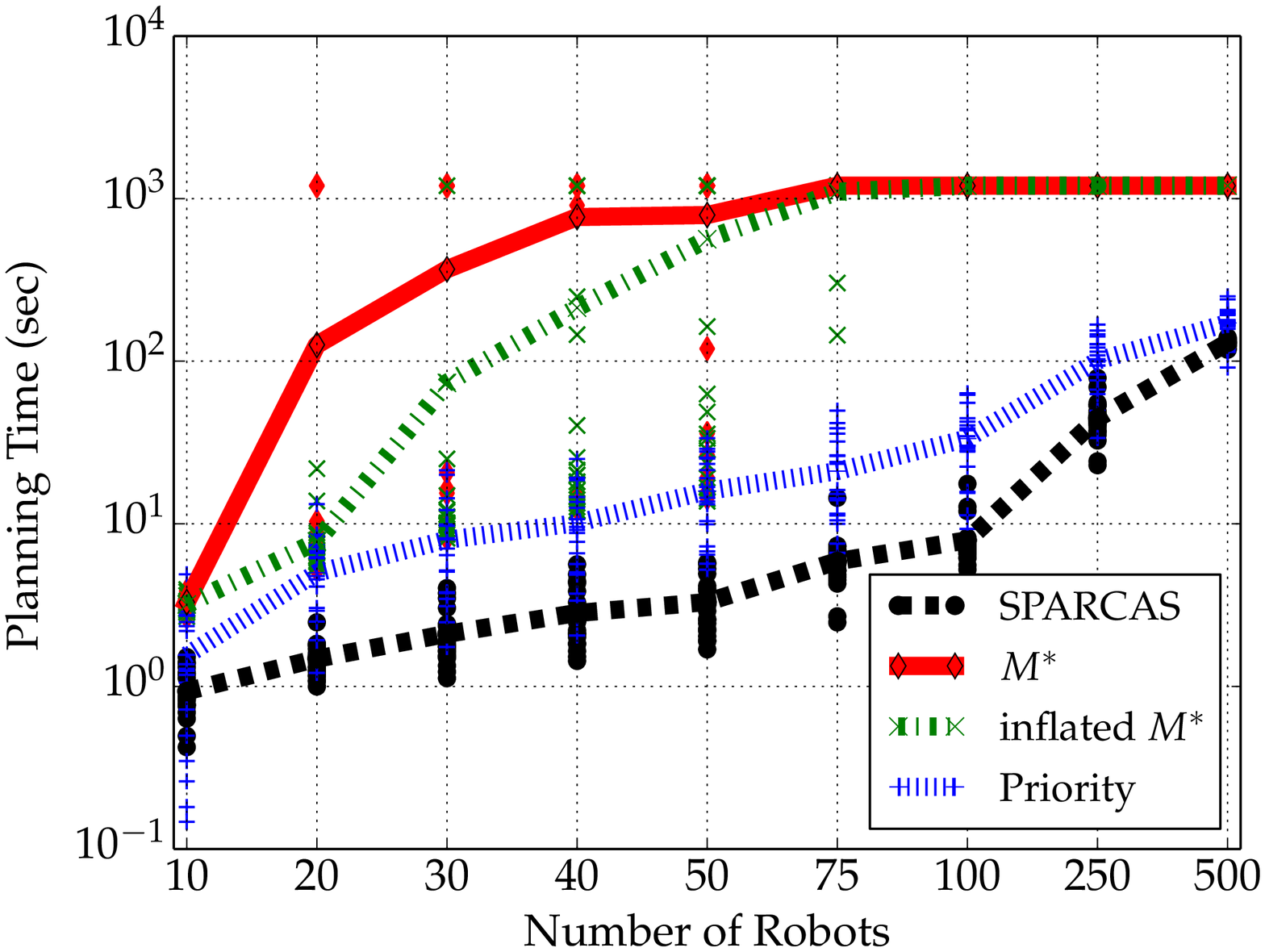}
 \caption{The y-axis shows the total planning time (in sec) of the different algorithms and the x-axis shows the number of robots. For \mech, the planning time is the sum of the offline path computation and online spot auction times. Workspace size $100 \times 100$.}
\label{fig:comparison-with-m*-sync}
\end{minipage}
\hspace{0mm}
\begin{minipage}{0.49\linewidth}
\centering
 \includegraphics[width=0.99\linewidth]{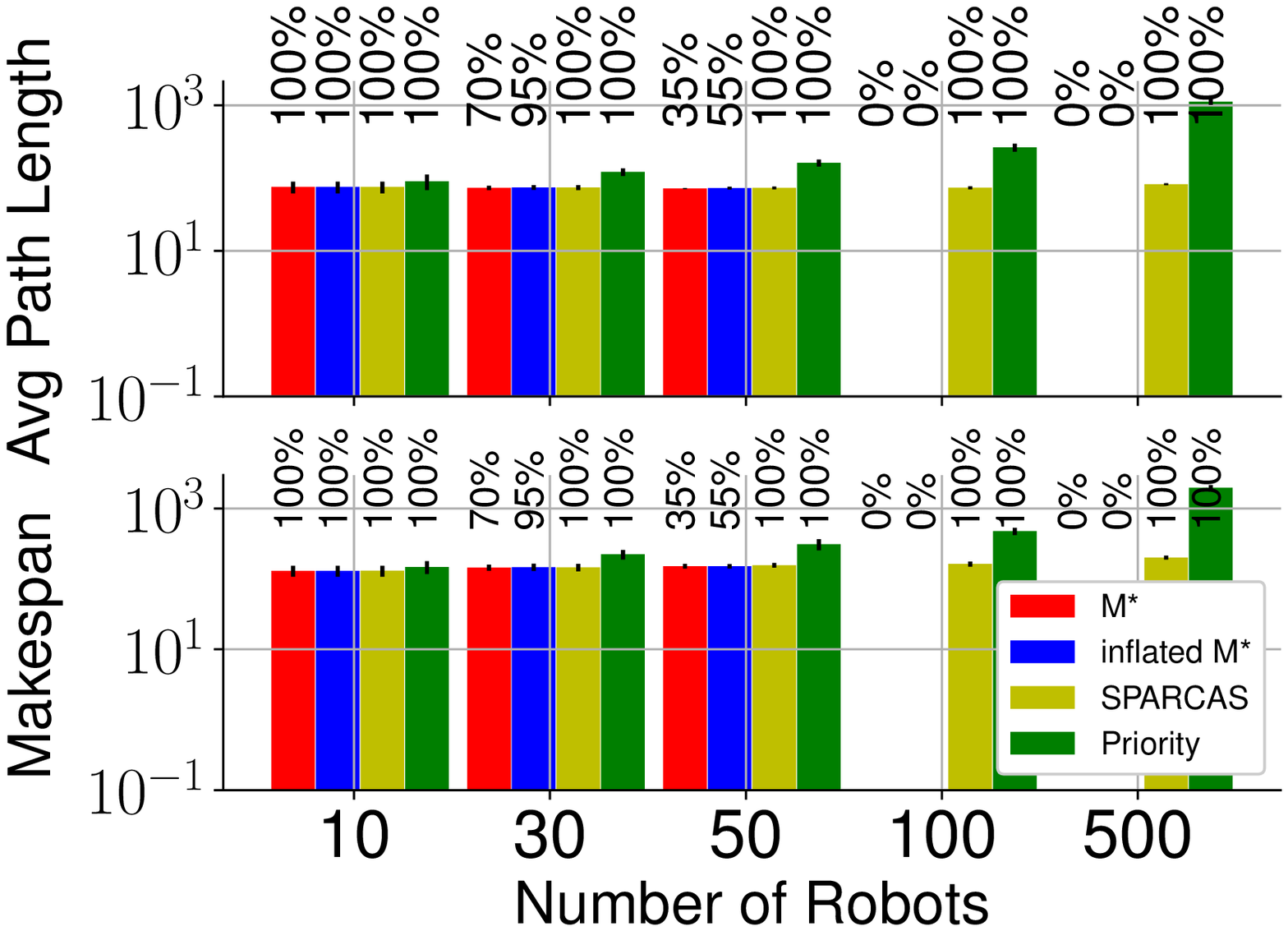}
 \caption{Average path length and makespan for different mechanisms. The numbers on the bar denote the percentage of cases where the mechanism could find a path within the timeout period.\sn{labels are almost invisible}}
 \label{fig:path_length}
\end{minipage}
\end{figure}

% \begin{figure}[t]
% \centering
% \begin{center}
% % \subfloat[Workspace size $50 \times 50$]
% % {\label{fig:plot1}\includegraphics[scale = 0.3]{images/plot_50.eps}} %\hspace*{-2.0em}
% % \subfloat[Workspace size $100 \times 100$]
% % {\label{fig:plot2}\includegraphics[width=0.85\linewidth]{images/sync_100.eps}} 
% {\label{fig:plot2}\includegraphics[width=0.85\linewidth]{images/planning_time_4.eps}}
% % \subfloat[Workspace size $150 \times 150$]
% % {\label{fig:plot3}\includegraphics[scale = 0.3]{images/plot_150.eps}}
% % \hspace{0in}
% % \\
% % \subfloat[Workspace size $200 \times 200$]
% % {\label{fig:plot4}\includegraphics[width=0.85\linewidth]{images/sync_200.eps}}
% \caption{The y-axis shows the total planning time (in sec) of the different algorithms and the x-axis shows the number of robots. For \mech, the planning time is the sum of the offline path computation and online spot auction times. Workspace size $100 \times 100$.}
% \label{fig:comparison-with-m*-sync}
% \end{center}
% \end{figure} 

We also compare the makespan (maximum time among all robots to reach destination) and average path execution time (the average of the individual path lengths) for the mechanisms for a workspace of size $100 \times 100$ in \Cref{fig:path_length}. The method of the plot is identical to \Cref{fig:comparison-with-m*-sync}. The numbers on top of the bars show the percentage of successful completion (not hitting timeout). The results show that there is not much difference in the execution part of these mechanisms and that \mech\ provides a path which is very close in length to that of the optimal path generated by M*.
% \begin{figure}[t]
% \centering
%  \includegraphics[width=0.79\linewidth]{images/path_log_scale_updated_5.eps}
%  \caption{Average path length and makespan for different mechanisms. The numbers on the bar denote the percentage of cases where the mechanism could find a path within the timeout period.\sn{labels are almost invisible}}
%  \label{fig:path_length}
% \end{figure}

% \stepcounter{paragraph}
% \paragraph{\arabic{section}.\arabic{paragraph} Delays of different priority classes}
\subsection{Delays of different priority classes} \label{sec:delay}
We compare how long robots of different classes (economy, regular, and premium) take to reach their destinations. The first and second subplots of \Cref{fig:wait_time} show the average waiting times and payments respectively of the different classes of robots for a workspace size of $100\times100$. 
We see that for a reasonable congestion, \mech\ prioritizes the higher classes for a greater payment.
% When the number of robots is small, there is little congestion (chance of possible collisions), which leads to all classes of robots having similar waiting times, i.e., no statistically significant difference in their delays. When the number of robots increase, leading to a high congestion, we see that the robots of higher priority classes have statistically significant smaller waiting times. The payments under \mech\ is higher for the higher priority classes as shown in the figure.
% 
\begin{figure}[h!]
\centering
 \includegraphics[width=0.80\linewidth]{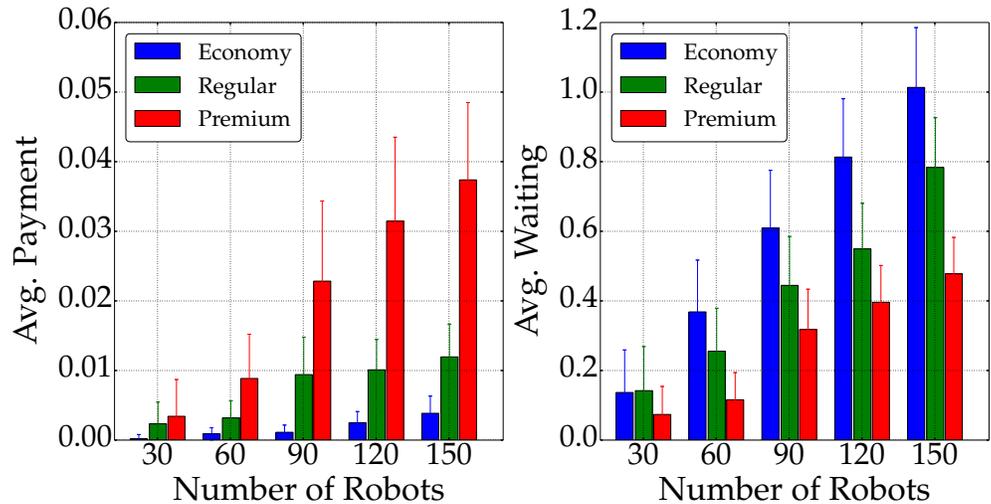}
 \caption{Average waiting time (in sec) and payments under \mech\ for different classes of robots.}
 \label{fig:wait_time}
\end{figure}

% \stepcounter{paragraph}
% \paragraph{\arabic{section}.\arabic{paragraph} Experiments with ROS}
\subsection{Experiments with ROS}
\label{sec:ros}
We have simulated \mech\ for up to 10 TurtleBots~\cite{turtlebot} on ROS~\cite{ROS}. 
We have used a workspace of size $14\times14$ (similar to \Cref{fig:fig-workspace}) with each grid cell of length $1\si{m}$. 
The linear and the angular velocity of the TurtleBots have been chosen in such a way that each motion primitive takes $6.5\si{s}$
for execution. 
The maximum time for running an auction (decides the length of a mini-slot) observed in all our experiments is about $60\si{\milli\second}$.
The video of our experiments is submitted as a supplementary  material.
% Since there is no provision of supplementary materials, we can share the video for our ROS experiments on request.

We also ran experiments to find out (a)~the payments under \mech, and (b)~the capability to handle dynamic arrival of robots. The payments are sufficiently small and the dynamic arrivals are gracefully handled in \mech. The details of these experiments are in the supplemental material.

\subsection{Capability of handling dynamic robot arrival} \label{sec:dynamic}
% {\color{red} TODO (Sankar): add a description of the experimental set up}.

% \noindent\emph{\textbf{Experimental Set-up}}: 
In this section, we study the robustness of \mech\ against dynamic robot arrivals.
The setup remains similar to the previous subsection -- we mention the differences as follows.
We partition the total number of robots into two groups of equal size for this evaluation. 
% while evaluating the capability of SPARCAS and M* for handling the dynamic robot arrivals. 
The first group of robots arrive at the beginning. The rest $50\%$ of the robots arrive independently and uniformly at random within the time interval of zero and the length of the workspace.
% \sn{can two robots choose the same point of arrival at the same time (which has positive probability)? how is that handled?} 
% Here, we assume that the maximum path length of any robot of the first group is at least as equal to the length of the workspace.\sn{is that required? if the first phase of robots depart early, the next batch of robots will experience less congestion, but that's completely a legitimate instance of this experiment. I don't think we place any such condition in our code.} 
In \mech, a newly arrived robot computes its own path by using the A* algorithm \cite{Hart68}, and starts to follow that path immediately. However, in M*, a newly arrived robot requests path to the centralized M* path planner. The M* planner collects such requests and waits for re-planning until the number of the newly arrived robots exceeds a predefined threshold. In our experiments, this threshold is set to $2$. During the re-planning, the M* planner considers the current locations of the previously arrived robots as their source locations and excludes the robots which already have reached their respective goal locations. For prioritized planning, the robots that arrive later are considered to have a lower priority than the robots arrived already, and compute their path considering the previous robots' positions as dynamic obstacles. \Cref{fig:comparison-with-m*-async} shows the results of the experiments for two different workspace sizes: $100 \times 100$ and $198 \times 198$. 
% The number of robots vary between $10$ to $50$, beyond which every instance of M* reaches a timeout. 
% 
\begin{figure}[!ht]
% \vspace{-0.2cm}
\begin{center}
\subfloat[Workspace size $100 \times 100$]
{\includegraphics[width=0.49\linewidth]{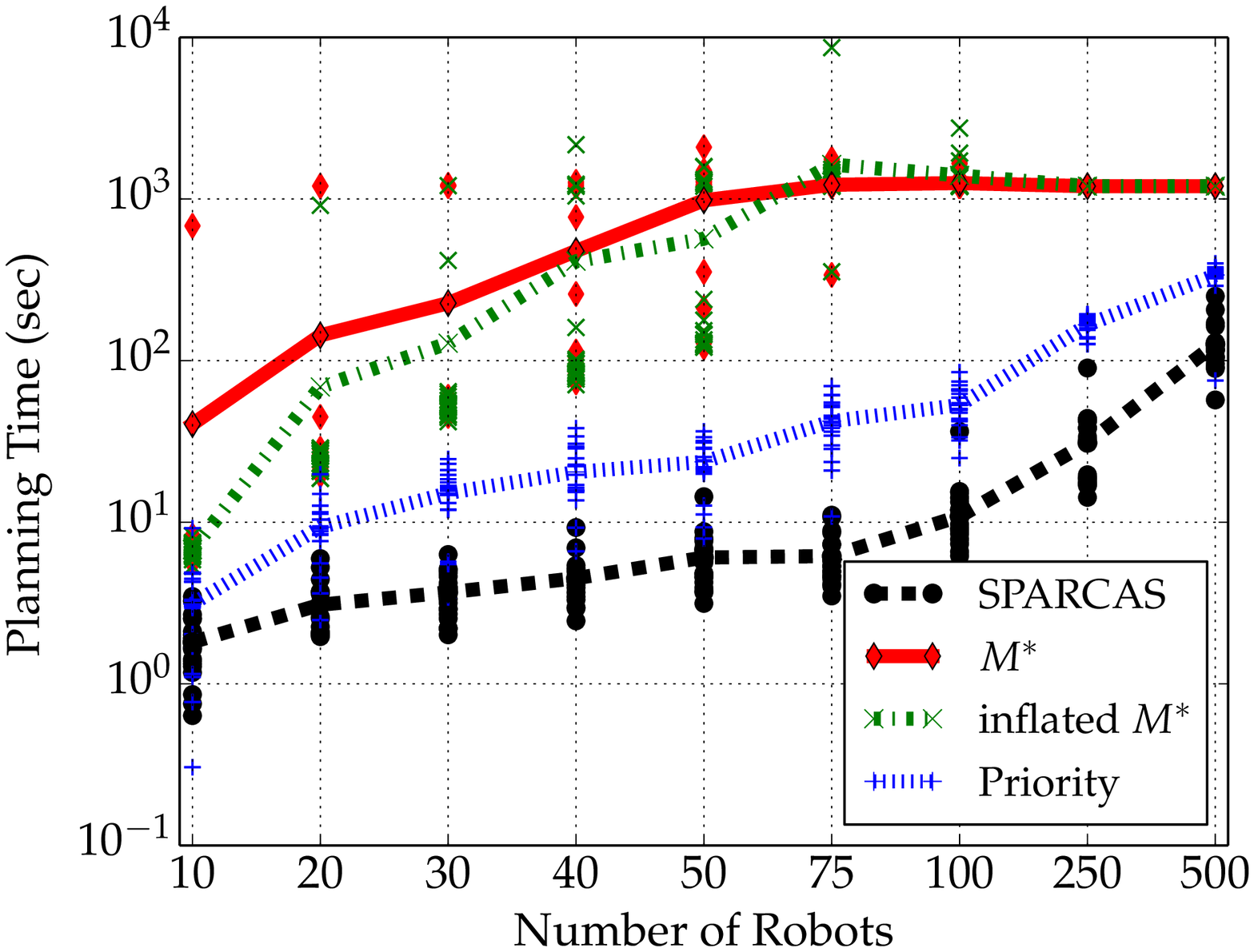}} \hspace*{0em} 
% \\
\subfloat[Workspace size $198 \times 198$]
{\includegraphics[width=0.49\linewidth]{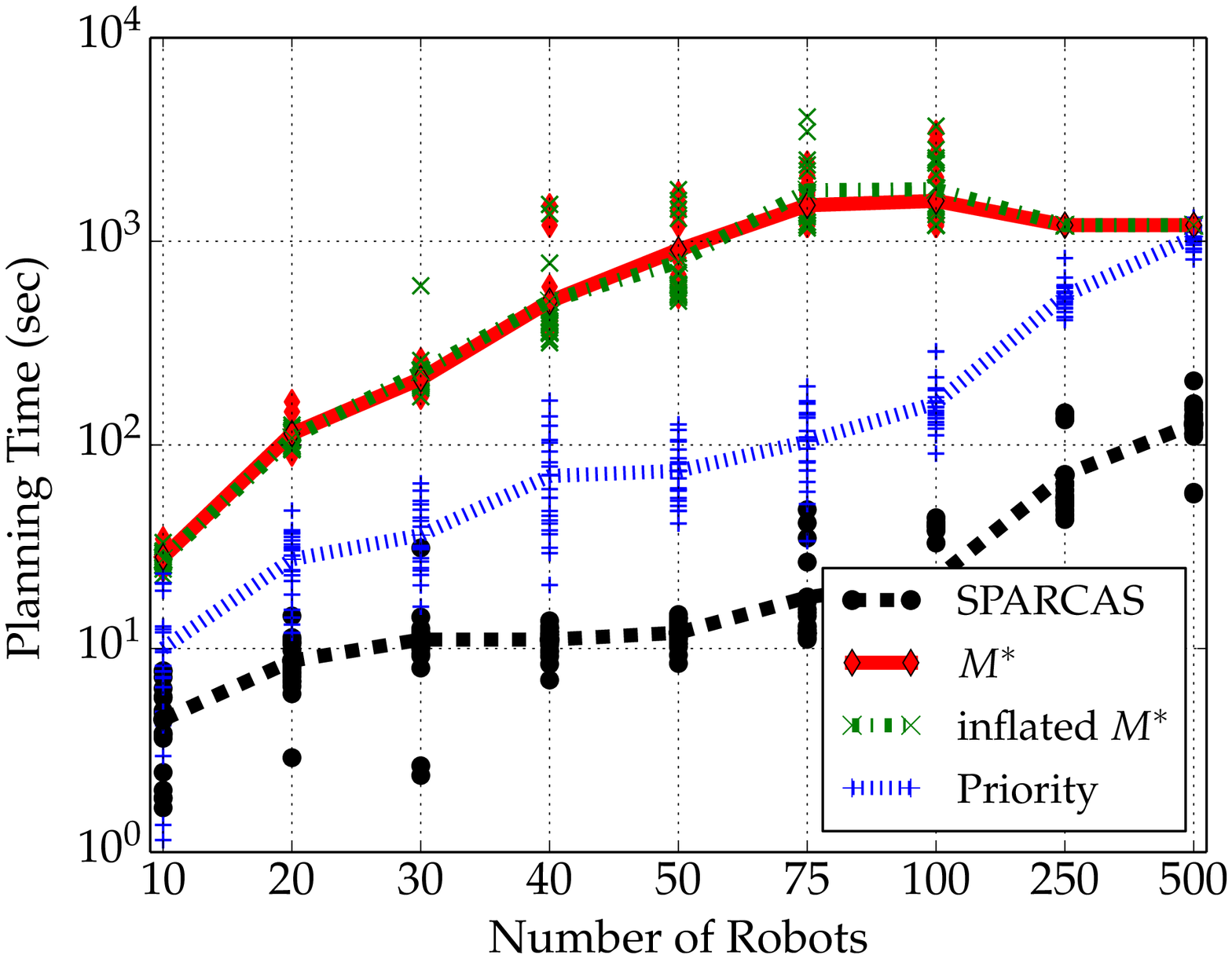}} 
\caption{The y-axis shows the total planning time (in sec) of the algorithms and the x-axis shows the number of robots. For \mech, the planning time is the sum of the offline path computation and online spot auction times.}
\label{fig:comparison-with-m*-async}
\end{center}
\end{figure}  

\subsection{Payments under \mech} \label{sec:payment}
\mech\ is different from the other collision avoiding mechanisms in its use of payments which ensure truthful participation of the competitive robots. Therefore, it is a natural question to find out how large the payments are in comparison with the valuations received by the robots. In our experiment with a workspace size of $100 \times 100$, 95.12\% of the robots were never needed to pay. For the number of robots: $10, 20, 30, 40, 50$ and $75$, the tuple of average valuation and average payment of the robots were $(1.9, 7.5\times10^{-4})$, $(2.16, 23.75\times10^{-4})$, $(2.2, 47.5\times10^{-4})$, $(2.33, 47.5\times10^{-4})$, $(2.41, 63\times10^{-4})$, and $(2.56, 134.7\times10^{-4})$ respectively. Similarly, another important question is whether the robots encounter a negative payoff, i.e., payment more than valuation. In the experiment, we find that the answer to this is negative. Hence the robots have incentives to voluntarily participate in this mechanism. We also find, quite expectedly as was argued in \Cref{foot:payment-positive}, that the payments by the robots are always non-negative.

\section{Discussions}
We presented a scalable decentralized collision avoidance mechanism for a {\em competitive} multi-robot system using ideas of mechanism design. We prove that it is truthful, efficient, deadlock-free, and budget-balanced. We exhibit experimentally that it is scalable to hundreds of robots, and can handle dynamic arrival without compromising the path-length optimality too much. In future, we would like to extend the algorithm to multi-intersection `traffic-jam's and conduct experiments on a real multi-robot system.

%{\balance
%\bibliographystyle{plainnat}
%\bibliography{references,references_robots,abb,ultimate,master}

\begin{thebibliography}{41}
\providecommand{\natexlab}[1]{#1}
\providecommand{\url}[1]{\texttt{#1}}
\expandafter\ifx\csname urlstyle\endcsname\relax
  \providecommand{\doi}[1]{doi: #1}\else
  \providecommand{\doi}{doi: \begingroup \urlstyle{rm}\Url}\fi

\bibitem[Amir et~al.(2015)Amir, Sharon, and Stern]{amir2015multi}
Ofra Amir, Guni Sharon, and Roni Stern.
\newblock Multi-agent pathfinding as a combinatorial auction.
\newblock In \emph{Twenty-Ninth AAAI Conference on Artificial Intelligence},
  2015.

\bibitem[Azarm and Schmidt(1997)]{Azarm97}
K.~Azarm and G~Schmidt.
\newblock Conflict-free motion of multiple mobile robots based on decentralized
  motion planning and negotiation.
\newblock In \emph{IEEE international conference on robotics and automation},
  volume~4, pages 3526--3533, 1997.

\bibitem[Bererton et~al.(2004)Bererton, Gordon, and Thrun]{bererton2004auction}
Curt Bererton, Geoffrey~J Gordon, and Sebastian Thrun.
\newblock Auction mechanism design for multi-robot coordination.
\newblock In \emph{Advances in Neural Information Processing Systems}, pages
  879--886, 2004.

\bibitem[Bogue(2016)]{warehouseRobotics}
R.~Bogue.
\newblock Growth in e-commerce boosts innovation in the warehouse robot market.
\newblock \emph{Industrial Robot: An International Journal}, 43\penalty0
  (6):\penalty0 583--587, 2016.

\bibitem[B{\"o}rgers(2015)]{borgers2015introduction}
Tilman B{\"o}rgers.
\newblock \emph{An introduction to the theory of mechanism design}.
\newblock Oxford University Press, USA, 2015.

\bibitem[Calliess et~al.(2011)Calliess, Lyons, and Hanebeck]{calliess2011lazy}
Jan-P Calliess, Daniel Lyons, and Uwe~D Hanebeck.
\newblock Lazy auctions for multi-robot collision avoidance and motion control
  under uncertainty.
\newblock In \emph{International Conference on Autonomous Agents and Multiagent
  Systems}, pages 295--312. Springer, 2011.

\bibitem[Chen et~al.(2017)Chen, Liu, Everett, and How]{chen2017decentralized}
Yu~Fan Chen, Miao Liu, Michael Everett, and Jonathan~P How.
\newblock Decentralized non-communicating multiagent collision avoidance with
  deep reinforcement learning.
\newblock In \emph{Robotics and Automation (ICRA), 2017 IEEE International
  Conference on}, pages 285--292. IEEE, 2017.

\bibitem[Chun et~al.(1999)Chun, Zheng, and Chang]{Chun99}
L.~Chun, Z.~Zheng, and W~Chang.
\newblock A decentralized approach to the conflict-free motion planning for
  multiple mobile robots.
\newblock In \emph{IEEE international conference on robotics and automation
  (ICRA)}, volume~2, pages 1544--1549, 1999.

\bibitem[Clarke(1971)]{Clar71}
E.~H. Clarke.
\newblock Multipart pricing of public goods.
\newblock \emph{Public Choice}, 11:\penalty0 17--33, 1971.

\bibitem[Desai et~al.(2017)Desai, Saha, Yang, Qadeer, and Seshia]{DesaiSYQS17}
Ankush Desai, Indranil Saha, Jianqiao Yang, Shaz Qadeer, and Sanjit~A. Seshia.
\newblock {DRONA:} a framework for safe distributed mobile robotics.
\newblock In \emph{Proceedings of the 8th International Conference on
  Cyber-Physical Systems, {ICCPS} 2017, Pittsburgh, Pennsylvania, USA, April
  18-20, 2017}, pages 239--248, 2017.

\bibitem[Desaraju and How(2012)]{Desaraju12}
V.R. Desaraju and J.~P. How.
\newblock Decentralized path planning for multi-agent teams with complex
  constraints.
\newblock \emph{Autonomous Robots}, 32\penalty0 (4):\penalty0 385--403, 2012.

\bibitem[Dhinakaran et~al.(2017)Dhinakaran, Chen, Chou, Shih, and
  Tomlin]{dhinakaran2017hybrid}
Aparna Dhinakaran, Mo~Chen, Glen Chou, Jennifer~C Shih, and Claire~J Tomlin.
\newblock A hybrid framework for multi-vehicle collision avoidance.
\newblock In \emph{2017 IEEE 56th Annual Conference on Decision and Control
  (CDC)}, pages 2979--2984. IEEE, 2017.

\bibitem[Erdmann and Lozano-Perez(1986)]{ErdmannRA86}
M.~Erdmann and T.~Lozano-Perez.
\newblock On multiple moving objects.
\newblock In \emph{ICRA}, volume~3, pages 1419--1424, 1986.

\bibitem[Garage(2011)]{turtlebot}
Willow Garage.
\newblock Turtlebot.
\newblock \emph{Website: http://turtlebot.com/}, 2011.

\bibitem[Gibbard(1973)]{Gib73}
A.~Gibbard.
\newblock Manipulation of voting schemes.
\newblock \emph{Econometrica}, 41:\penalty0 587--602, 1973.

\bibitem[Groves(1973)]{Grov73}
T.~Groves.
\newblock Incentives in teams.
\newblock \emph{Econometrica}, 41:\penalty0 617--631, 1973.

\bibitem[Guizzo(2008)]{warehouseRoboticsGuizzo}
E.~Guizzo.
\newblock Three engineers, hundreds of robots, one warehouse.
\newblock \emph{IEEE Spectrum}, 45\penalty0 (7):\penalty0 26--34, 2008.

\bibitem[Hart et~al.(1968)Hart, Nilsson, and Raphael]{Hart68}
P.~Hart, N.~Nilsson, and B.~Raphael.
\newblock A formal basis for the heuristic determination of minimum cost paths.
\newblock \emph{IEEE Trans. Syst. Sci. Cybern.}, 4\penalty0 (2):\penalty0
  100--107, 1968.

\bibitem[Hoffmann and Tomlin(2008)]{Hoffmann08}
G.~Hoffmann and C.~Tomlin.
\newblock Decentralized cooperative collision avoidance for acceleration
  constrained vehicles.
\newblock In \emph{IEEE conference on decision and control (CDC)}, pages
  4357--4363, 2008.

\bibitem[Jager and Nebel(2001)]{Jager01}
M.~Jager and B.~Nebel.
\newblock Decentralized collision avoidance, deadlock detection, and deadlock
  resolution for multiple mobile robots.
\newblock In \emph{IEEE/RSJ international conference on intelligent robots and
  systems (IROS)}, volume~3, pages 1213--1219, 2001.

\bibitem[Lagoudakis et~al.(2005)Lagoudakis, Markakis, Kempe, Keskinocak,
  Kleywegt, Koenig, Tovey, Meyerson, and Jain]{lagoudakis2005auction}
Michail~G Lagoudakis, Evangelos Markakis, David Kempe, Pinar Keskinocak,
  Anton~J Kleywegt, Sven Koenig, Craig~A Tovey, Adam Meyerson, and Sonal Jain.
\newblock Auction-based multi-robot routing.
\newblock In \emph{Robotics: Science and Systems}, volume~5, pages 343--350.
  Rome, Italy, 2005.

\bibitem[Nunes and Gini(2015)]{nunes2015multi}
Ernesto Nunes and Maria Gini.
\newblock Multi-robot auctions for allocation of tasks with temporal
  constraints.
\newblock In \emph{Twenty-Ninth AAAI Conference on Artificial Intelligence},
  2015.

\bibitem[Olfati-Saber et~al.(2007)Olfati-Saber, Fax, and
  Murray]{Olfati-Saber07}
R.~Olfati-Saber, J.~Fax, and R.~Murray.
\newblock Consensus and cooperation in networked multi-agent systems.
\newblock \emph{Proceedings of the IEEE}, 95\penalty0 (1):\penalty0 215--233,
  2007.

\bibitem[Pallottino et~al.(2004)Pallottino, Scordio, and Bicchi]{Pallottino04}
L.~Pallottino, V.~Scordio, and A~Bicchi.
\newblock Decentralized cooperative conflict resolution among multiple
  autonomous mobile agents.
\newblock In \emph{IEEE conference on decision and control (CDC)}, volume~5,
  pages 4758--4763, 2004.

\bibitem[Purwin et~al.(2008)Purwin, D’Andrea, and Lee]{Purwin08}
O.~Purwin, R.~D’Andrea, and J~Lee.
\newblock Theory and implementation of path planning by negotiation for
  decentralized agents.
\newblock \emph{Robotics and Autonomous Systems}, 56\penalty0 (5):\penalty0
  422--436, 2008.

\bibitem[Quigley et~al.(2009)Quigley, Gerkey, Conley, Faust, Foote, Leibs,
  Berger, Wheeler, and Ng]{ROS}
M.~Quigley, B.~Gerkey, K.~Conley, J.~Faust, T.~Foote, J.~Leibs, E.~Berger,
  R.~Wheeler, and A.~Y. Ng.
\newblock {ROS}: an open-source robot operating system.
\newblock In \emph{Open-Source Software workshop of the International
  Conference on Robotics and Automation (ICRA)}, 2009.

\bibitem[Roberts(1979)]{Roberts79}
Kevin Roberts.
\newblock \emph{{The Characterization of Implementable Choice Rules}}, chapter
  {Aggregation and Revelation of Preferences}, pages 321--348.
\newblock North Holland Publishing, 1979.

\bibitem[Saha et~al.(2014)Saha, Ramaithitima, Kumar, Pappas, and
  Seshia]{SahaRKPS14}
Indranil Saha, Rattanachai Ramaithitima, Vijay Kumar, George~J. Pappas, and
  Sanjit~A. Seshia.
\newblock Automated composition of motion primitives for multi-robot systems
  from safe {LTL} specifications.
\newblock In \emph{2014 {IEEE/RSJ} International Conference on Intelligent
  Robots and Systems, Chicago, IL, USA, September 14-18, 2014}, pages
  1525--1532, 2014.

\bibitem[Saha et~al.(2016)Saha, Ramaithitima, Kumar, Pappas, and
  Seshia]{SahaRKPS16}
Indranil Saha, Rattanachai Ramaithitima, Vijay Kumar, George~J. Pappas, and
  Sanjit~A. Seshia.
\newblock Implan: Scalable incremental motion planning for multi-robot systems.
\newblock In \emph{7th {ACM/IEEE} International Conference on Cyber-Physical
  Systems, {ICCPS} 2016, Vienna, Austria, April 11-14, 2016}, pages
  43:1--43:10, 2016.

\bibitem[Satterthwaite(1975)]{Sat75}
M.~Satterthwaite.
\newblock Strategy-proofness and {A}rrow's conditions: Existence and
  correspondence theorems for voting procedures and social welfare functions.
\newblock \emph{Journal of Economic Theory}, 10:\penalty0 187--217, 1975.

\bibitem[Shoham and Leyton-Brown(2008)]{SL08}
Y.~Shoham and K.~Leyton-Brown.
\newblock \emph{Multiagent Systems: Algorithmic, Game-Theoretic, and Logical
  Foundations}.
\newblock Cambridge University Press, 2008.

\bibitem[Snape et~al.(2010)Snape, Van Den~Berg, Guy, and
  Manocha]{snape2010smooth}
Jamie Snape, Jur Van Den~Berg, Stephen~J Guy, and Dinesh Manocha.
\newblock Smooth and collision-free navigation for multiple robots under
  differential-drive constraints.
\newblock In \emph{Intelligent Robots and Systems (IROS), 2010 IEEE/RSJ
  International Conference on}, pages 4584--4589. IEEE, 2010.

\bibitem[Takei et~al.(2012)Takei, Huang, Ding, and Tomlin]{takei2012time}
Ryo Takei, Haomiao Huang, Jerry Ding, and Claire~J Tomlin.
\newblock Time-optimal multi-stage motion planning with guaranteed collision
  avoidance via an open-loop game formulation.
\newblock In \emph{2012 IEEE International Conference on Robotics and
  Automation}, pages 323--329. IEEE, 2012.

\bibitem[Turpin et~al.(2014)Turpin, Michael, and Kumar]{TurpinMK14}
Matthew Turpin, Nathan Michael, and Vijay Kumar.
\newblock Capt: Concurrent assignment and planning of trajectories for multiple
  robots.
\newblock \emph{International Journal of Robotic Research}, 33\penalty0 (1),
  2014.

\bibitem[van~den Berg and Overmars(2005)]{Berg05}
J.~van~den Berg and M.~Overmars.
\newblock Prioritized motion planning for multiple robots.
\newblock In \emph{IEEE/RSJ International Conference on Intelligent Robots and
  Systems (IROS)}, pages 2217--2222, 2005.

\bibitem[Velagapudi et~al.(2010)Velagapudi, Sycara, and Scerri]{Velagapudi10}
P.~Velagapudi, K.~Sycara, and P.~Scerri.
\newblock Decentralized prioritized planning in large multirobot teams.
\newblock In \emph{IEEE/RSJ international conference on intelligent robots and
  systems (IROS)}, pages 4603--4609, 2010.

\bibitem[Vickrey(1961)]{Vick61}
W.~Vickrey.
\newblock Counter speculation, auctions, and competitive sealed tenders.
\newblock \emph{Journal of Finance}, 16\penalty0 (1):\penalty0 8--37, 1961.

\bibitem[Wagner and Choset(2011)]{WagnerC11}
Glenn Wagner and Howie Choset.
\newblock M$^*$: A complete multirobot path planning algorithm with performance
  bounds.
\newblock In \emph{IROS}, pages 3260--3267, 2011.

\bibitem[Wagner and Choset(2015)]{WAGNER20151}
Glenn Wagner and Howie Choset.
\newblock Subdimensional expansion for multirobot path planning.
\newblock \emph{Artificial Intelligence}, 219:\penalty0 1 -- 24, 2015.

\bibitem[Wulfraat()]{kivaWarehouseRobots}
M.~Wulfraat.
\newblock Is {K}iva systems a good fit for your distribution center? an
  unbiased distribution consultant evaluation.
\newblock \url{http://www.mwpvl.com/html/kiva_systems.html}.
\newblock Accessed: October 2016.

\bibitem[Yu and La{V}alle(2013)]{JL13}
Jingjin Yu and Steven~M. La{V}alle.
\newblock Planning optimal paths for multiple robots on graphs.
\newblock In \emph{ICRA}, pages 3612--3617, 2013.

\end{thebibliography}
%}

\end{document}